\documentclass[letterpaper, 10 pt, conference]{ieeeconf}  

\IEEEoverridecommandlockouts                              
\usepackage{amssymb}
\usepackage{algorithm}
\usepackage{algorithmic}
\usepackage{bm}
\usepackage{dsfont}
\usepackage{bbm}

\usepackage{cite}

\usepackage{graphicx}

\newtheorem{corollary}{Corollary}
\usepackage[utf8]{inputenc}
\usepackage[english]{babel}
 
\newtheorem{theorem}{Theorem}
\newtheorem{remark}{\textbf{Remark}}

\usepackage{authblk}

\usepackage{algorithm,algorithmic}
\usepackage{caption}

\newlength\myindent
\newcommand\bindent{%
  \begingroup
  \setlength{\itemindent}{\myindent}
  \addtolength{\algorithmicindent}{\myindent}
}
\newcommand\eindent{\endgroup}

\usepackage{amsmath}
\DeclareMathOperator*{\argmax}{argmax} 
\DeclareMathOperator*{\argmin}{argmin} 

\newtheorem{example}{Example}
\newtheorem{definition}{Definition}

\title{\LARGE \bf
Risk-Averse Explore-Then-Commit Algorithms for Finite-Time Bandits
}

\author[1]{Ali Yekkehkhany}
\author[1]{Ebrahim Arian}
\author[2]{Mohammad Hajiesmaili}
\author[1]{Rakesh Nagi}
\affil[1]{University of Illinois at Urbana-Champaign}
\affil[2]{University of Massachusetts Amherst}
\date{}                     
\begin{document}

\maketitle
\thispagestyle{empty}
\pagestyle{plain}

\begin{abstract}

In this paper, we study multi-armed bandit problems in an explore-then-commit setting. In our proposed explore-then-commit setting, the goal is to identify the best arm after a pure experimentation (exploration) phase and exploit it once or for a given finite number of times. We identify that although the arm with the highest expected reward is the most desirable objective for infinite exploitations, it is not necessarily the one that is most probable to have the highest reward in a single or finite-time exploitations. Alternatively, we advocate the idea of risk–aversion where the objective is to compete against the arm with the best risk–return trade–off. We propose two algorithms whose objectives are to select the arm that is most probable to reward the most.
Using a new notion of finite-time exploitation regret, we find an upper bound of order $\ln \left ( \frac{1}{\epsilon} \right )$ for the minimum number of experiments before commitment, to guarantee  upper bound $\epsilon$ for regret. As compared to existing risk-averse bandit algorithms, our algorithms do not rely on hyper-parameters, resulting in a more robust behavior,
which is verified by numerical evaluations.


\end{abstract}

\section{Introduction}
\label{introduction}
One of the classes of decision making models is the multi-armed bandit (MAB) framework where decision makers learn the model of different arms that are unknown and actions do not change the state of arms \cite{vermorel2005multi}.
The MAB problem was originally proposed by Robbins \cite{robbins1952some}, and has a wide range of applications in finance \cite{bergemann1998dynamic, bergemann2005financing}, health-care \cite{zois2016sequential}, autonomous vehicles \cite{musavi2016unmanned, musavi2016game}, communication and networks \cite{avner2016multi, yekkehkhany2019blind, xie2016scheduling, yekkehkhany2017near, yekkehkhany2017gb}, and energy management \cite{galichet2013exploration, maghsudi2017distributed} to name but a few. In the classical MAB problem, the decision-maker sequentially selects an arm (action) with an unknown reward distribution out of $K$ independent arms. 
The random reward of the selected arm is revealed and the rewards of other arms remain unknown.
At each step, the decision-maker encounters a dilemma between exploitation of the best identified arm versus exploration of alternative arms.
The goal of the classical model of multi-armed bandit is to maximize the expected cumulative reward over a time horizon.


In this paper, we focus on a setting where a player is allowed to explore different arms in the  exploration (or experimentation, used interchangeably) phase before committing to the best identified arm for exploitation in one or a given finite number of times.
This setting of interest is motivated by several application domains such as personalized health-care and one-time investment. In such applications, exploitation is costly and/or it is infeasible to exploit for a large number of times, but arms can be experimented by simulation and/or based on the historical data for multiple times with negligible cost \cite{bui2011committing}. The big step in personalized health-care is to provide an individual patient with his/her disease risk profile based on his/her electronic medical record and personalized assessments \cite{chawla2013bringing, pritchard2017strategies}.
The different treatments (arms) are evaluated for a person by simulation or mice trials for many times with a low cost, but one personalized treatment is exploited once  for a patient in the end \cite{priyanka2014survey, abrahams2005personalized}. Another example of one-time exploitation is one-time investment where an investor chooses a factory out of multiple ones. Based on experimentation on historical data, he/she selects a factory to invest in once. The common theme in both above examples is to identify the best arm for one-time exploitation after an experimentation phase of pure exploration.

The above setting falls in the class of MAB problems called \textit{explore-then-commit}.
The previous works \cite{bui2011committing, garivier2016explore, garivier2018explore, prashanth2018cs6046, galichet2013exploration, sani2012risk} on explore-then-commit bandits, to the best of our knowledge, try to identify the arm with \textit{an optimum risk-return criterion} on an expectation sense up to a hyper-parameter. Even though this objective is desirable in the settings with infinite exploitations, it is not necessarily the best objective in the explore-then-commit setting with a single or finite exploitations.
We further elaborate on this observation by an illustrative example in Section~\ref{Problem_Statement}.
We advocate an alternative approach in which the objective is to \textit{select an arm that is most probable to reward the most}. It has been realized that in many scenarios of multi-armed bandits, considering maximum expected reward as an objective to select an arm is not the best strategy. In such scenarios, players not only aim to achieve the maximum cumulative reward, but they also want to minimize the uncertainty such as risk in the outcome~\cite{vakili2016risk}, and scuh approaches are known as \textit{risk-averse} MAB.
In literature, there are several approaches to address the risk-averse MAB including mean-variance (MV) \cite{sani2012risk} and the conditional value at risk (CVaR) \cite{galichet2013exploration}. The performance of both MV and CVaR, are highly dependent on different single scalar hyper-parameters, and selecting an inappropriate hyper-parameter might degrade the performance substantially. More details on MV and CVaR criteria are given in Section~\ref{related_work}, and the negative impact of hyper-parameter mismatch is studied in Section~\ref{results}.


\textbf{Contributions:} We propose a class of hyper-parameter-free risk-averse algorithms (called \texttt{OTE/FTE-MAB}) for explore-then-commit bandits with finite-time exploitations. The goal of the algorithms is to select the arm that is most probable to give the player the highest reward. To analyze the algorithms, we define a new notion of finite-time exploitation regret for our setting of interest.
We provide concrete mathematical support to obtain an upper bound of order $\ln \left ( \frac{1}{\epsilon} \right )$ for the minimum number of experiments that should be done to guarantee upper bound $\epsilon$ for regret. More specifically, our results show that 
by utilizing the proposed algorithms, the regret can be bounded arbitrarily small by sufficient number of experimentations. As a salient feature, the \texttt{OTE/FTE-MAB} algorithm is hyper-parameter-free, so it is not prone to errors due to hyper-parameter mismatch.


\textbf{Organization of the Paper:}
Section \ref{related_work} discusses related work.
In Section \ref{Problem_Statement}, the one/finite-time exploitation multi-armed bandit problem after an experimentation phase is formally described. We define a new notion of one/finite-time exploitation regret for our problem setup. An example is provided clarifying the motivation of our work. In Section \ref{section_algo}, we propose the OTE-MAB and FTE-MAB algorithms, and find an upper bound of order $\ln \left ( \frac{1}{\epsilon} \right )$ for the minimum number of pure explorations needed to guarantee upper bound $\epsilon$ for regret.
In Section \ref{results}, we evaluate the OTE-MAB algorithm versus risk-averse baselines and compare the minimum number of experiments needed to guarantee an upper bound on regret for both the OTE-MAB and FTE-MAB algorithms. We conclude the paper with a discussion of opportunities for future work in Section \ref{conclusion_future}.

\section{Related Work}
\label{related_work}
Explore-then-commit bandit is a class of multi-armed bandit problems that has two consecutive phases named as exploration (experimentation) and commitment.
The decision maker can arbitrarily explore each arm in the experimentation phase; however, he/she needs to commit to one selected arm in the commitment phase.
There are several studies on explore-then-commit bandits in the literature as follows.
Bui et al. \cite{bui2011committing} studied the optimal number of explorations when cost is incurred in both phases.
Liau et al. \cite{liau2017stochastic} designed an explore-then-commit algorithm for the case where there is a limited space to record the arm reward statistics.
Perchet et al. \cite{perchet2016batched} studied explore-then-commit policy under the assumption that the employed policy must split explorations into a number of batches. None of these works have addressed the risk-averse issue on explore-then-commit bandits. In the following, we present an overview on risk-averse bandits.

There are several criteria to measure and to model risk in a risk-averse multi-armed bandit problem. One of the common risk measurements is the mean-variance paradigm \cite{markowitz1952portfolio}.
The two algorithms MV-LCB and ExpExp proposed by Sani et al. \cite{sani2012risk} are based on mean-variance concept.
They define the mean-variance of an arm with mean $\mu$ and variance $\sigma^2$ as MV$= \sigma^2 - \rho \cdot \mu$, where $\rho \geq 0$ is the absolute risk tolerance coefficient.
In an infinite horizon multi-armed bandit problem, MV-LCB plays the arm with minimum lower confidence bound for estimation of MV.
In a best-arm identification setting, the ExpExp algorithm explores each of the arms for the same number of times and selects the arm with minimum estimated MV.
This approach is followed by numerous researchers in risk-averse multi-armed bandit problems \cite{vakili2016risk, vakili2015mean, yu2013sample, vakili2015risk}.

Another way of considering risk in multi-armed bandit problems is to use conditional value at risk level $\alpha$, CVaR$_\alpha$, where it is the expected policy return in a specified quantile.
CVaR$_\alpha$ is utilized by Galichet et al. \cite{galichet2013exploration} in risk-aware multi-armed bandit problems.
They presented the Multi-Armed Risk-Aware Bandit (MaRaB) algorithm aiming to select the arm with the maximum conditional value at risk level $\alpha$, CVaR$_\alpha$.
Formally, let $0 < \alpha < 1$ be the target quantile level and $v_\alpha$ defined as $P(R < v_\alpha) = \alpha$ be the associated quantile value, where $R$ is the arm reward. The conditional value at risk $\alpha$ is then defined as CVaR$_\alpha = \mathds{E} \left [R | R < v_\alpha \right ]$.
CVaR$_\alpha$ is also followed by researchers in multi-armed bandit problems \cite{vakili2016risk, xu2018index, galichet2015contributions, cassel2018general, kolla2019risk}.
\section{Problem Statement}
\label{Problem_Statement}
Consider arms $\mathcal{K} = \{1, 2, \dots, K\}$ whose rewards are random variables $R_1, R_2, \dots, R_K$ that have unknown distributions $f_1, f_2, \dots, f_K$ with unknown finite expected values $\mu_1, \mu_2, \dots, \mu_K$, respectively. The goal is to identify the best arm at the end of an experimentation phase that is followed by an exploitation phase, where the best arm is exploited for a given number of times, $M < \infty$.
In the experimentation phase, each arm is sampled for $N$ independent times. Denote the observed reward of arm $k \in \mathcal{K}$ at iteration $n \in \{1, 2, \dots, N\}$ of experimentation by $r_{k, n}$.

Let $R_k^M = X_1 + X_2 + \dots + X_M,$ where $X_1, X_2, \dots, X_M$ are independent and identically distributed random variables and $X_1 \sim f_k$.
The optimum arm for $M$ exploitations in the sense that maximizes the probability of receiving the highest reward is
\begin{equation}
k^* = \underset{k}{\argmax} \ P(R_k^M \geq \boldsymbol{R}_{-k}^M),
\label{k_M_star}
\end{equation}
where $\boldsymbol{R}_{-k}^M = \{R_1^M, R_2^M, \dots, R_{k - 1}^M, R_{k + 1}^M, \dots, R_K^M\}$ and what $R_k^M$ being greater than or equal to a vector means is that it is greater than or equal to all elements of the vector.
Let $p_k^M = P(R_k^M \geq \boldsymbol{R}_{-k}^M)$. Given the above preliminaries, the finite-time exploitation regret is defined below.
\begin{definition}
The finite-time exploitation regret, $r_M(\Delta p)$, is defined as a function of an input $0 < \Delta p < 1$ for the selected arm $\hat{k}$ as
\begin{equation}
r_M(\Delta p) = P \left ( p_{k^*}^M - p_{\hat{k}}^M \geq \Delta p \right ).
\label{finite_time_exploitation_regret}
\end{equation}
\label{regret_definition}
\end{definition}
Note that the above definition of regret is different from the commonly used regret in bandit problems.
In the following, an example is presented that motivates to define this new notion of regret for the finite-time exploitation setting.

\subsection{Illustrative Example}
As mentioned in the Introduction, although the arm with the highest expected reward is the optimum arm for utilization in infinite number of exploitations, it is not necessarily the one that is most probable to have the highest reward in a single or some finite number of exploitations.
In the following example, two arms are considered such that $\mu_2 > \mu_1$, but it is more probable that a one-time exploitation of the first arm rewards us more than a one-time exploitation of the second arm.
Hence, arm $\underset{k}{\argmax} \ \mu_k$ is not necessarily the ideal arm for one-time exploitation let alone the arm with the maximum empirical mean, i.e. $\underset{k}{\argmax} \ \frac{\sum_{n = 1}^N r_{k, n}}{n}$.
\begin{example}
Consider two arms with the following independent reward distributions:
\[
\begin{aligned}
&f_1(u) = \alpha e^{-2(u - 3)^2} \cdot \mathds{1} \{ 0 \leq u \leq 10 \} \\
& f_2(u) = \beta \left ( 3e^{-8(u - 1)^2} + 2e^{-8(u - 8)^2} \right ) \cdot \mathds{1} \{ 0 \leq u \leq 10 \},
\end{aligned}
\]
where $\alpha$ and $\beta$ are constants for which each of the two distributions integrate to one and $\mathds{1}\{.\}$ is the indicator function.
\label{example1}
\end{example}

In example \ref{example1}, although the second arm has a larger mean than the first one, $\mu_2 \approx 3.8$ and $\mu_1 \approx 3$, the variance of reward received from the second arm is larger than that from the first one, which increases the risk of choosing the second arm for a one-time exploitation application. In fact, the first arm with lower mean is more probable to reward us more than the second arm since $P(R_1 \geq R_2) \approx 0.6 > 0.5$.
In general, a larger variance for the received reward is against the principle of risk-aversion where the objective is to keep a balance in a trade-off between the expected return and risk of an action \cite{sani2012risk}. 
Mean-variance is an existing approach to tackle this scenario. However, it has some drawbacks that are explained in details in the following.


The mean-variance (MV) of an arm depends on the hyper-parameter $\rho \geq 0$, which is the absolute risk tolerance coefficient.
The trade-off on $\rho$ is that if it is set to zero, the arm with the minimum variance is selected. On the other hand, if $\rho$ goes to infinity, the arm with the maximum expected reward is selected, which is the same as classical multi-armed bandit approach. Although the behavior of mean-variance trade-off is known for marginal values of $\rho$, it is not obvious what value of the hyper-parameter $\rho$ keeps a desirable balance between return and risk.
The choice of this hyper-parameter can be tricky and as will be shown in Section \ref{results}; a bad choice can increase the regret dramatically. As a simple example, consider two arms with unknown parameters $\mu_1 = 10, \sigma_1^2 = 10, \mu_2 = 1, \sigma_2^2 = 1$, and $P(R_1 > R_2) = 1$. The mean-variance trade-off is formalized as $\hat{\sigma}_k^2 - \rho \hat{\mu}_k$, where $\hat{\sigma}_k^2$ and $\hat{\mu}_k$ are empirical estimates of variance and mean of each arm. Note that the empirical means and variances converge to true values, so the second arm that is performing worse with probability one is selected if $\rho < 1$. In order to address this issue, we alternatively propose the following best arm identification algorithm for One-Time (Finite-time) Exploitation in a Multi-Armed Bandit problem (OTE/FTE-MAB algorithm) that has concrete mathematical support for its action and is hyper-parameter-free.

\section{One/Finite-Time Exploitation in Multi-Armed Bandit Problems after an Experimentation Phase}
\label{section_algo}
In this section, we propose the OTE-MAB and FTE-MAB algorithms.
The OTE-MAB algorithm is a specific case of FTE-MAB algorithm. Since the proof of theorem related to the FTE-MAB algorithm is notationally heavy, we first propose the OTE-MAB algorithm in Subsection \ref{sub_OTE_MAB} and postpone the FTE-MAB algorithm to Subsection \ref{sub_FTE_MAB}.

\subsection{The OTE-MAB Algorithm}
\label{sub_OTE_MAB}
The OTE-MAB algorithm desires to play the arm that is most probable to reward the most for the case $M = 1$ as
\begin{equation}
k^* = \underset{k}{\argmax} \ P(R_k \geq \boldsymbol{R}_{-k}),
\label{MMMMM}
\end{equation}
which is a specific case of Equation \eqref{k_M_star}. Due to simplicity of notation, the $M$-notation is eliminated in this subsection.
\begin{remark}
A more general version of the OTE-MAB algorithm is to concatenate a constant $c$ to vector $\boldsymbol{R}_{-k}$ as $\boldsymbol{R}_{-k} = \{R_1, R_2, \dots, R_{k - 1}, R_{k + 1}, \dots, R_K, c\}$.
\end{remark}

\begin{algorithm}[t]
\caption{The OTE-MAB Algorithm}
\begin{algorithmic} 
\STATE \textbf{Input} $0 < \epsilon_r, \Delta p < 1$
\STATE choose $N \geq \frac{2 \ln \left ( \frac{2K}{\epsilon_r} \right )}{{\Delta p}^2}$
\STATE \textbf{Experimentation Phase:}
\bindent
\FOR{$n = 1$ to $N$}
\STATE $r_{k, n}$ is observed for all $k \in \mathcal{K}$
\ENDFOR
\eindent
\IF{arms are independent}
\STATE Calculate $\hat{p}_k = \frac{\sum_{n_1 = 1}^N \sum_{n_2 = 1}^N \dots \sum_{n_K = 1}^N \mathds{1}\{ r_{k, n_k} \geq \boldsymbol{r}_{-k, n_{-k}} \} }{N^K}$
\ELSE
\STATE Calculate $\hat{p}_k = \frac{\sum_{n = 1}^N  \mathds{1}\{ r_{k, n} \geq \boldsymbol{r}_{-k, n} \} }{N}$
\ENDIF
\STATE \textbf{One-Time Exploitation:}
\bindent
\STATE Play arm $\hat{k} = \underset{k}{\argmax} \ \hat{p}_k$.
\eindent
\end{algorithmic}
\label{OTE-MAB}
\end{algorithm}

Since the reward distributions of the $K$ independent arms are not known, the exact values of $p_k = P(R_k \geq \boldsymbol{R}_{-k})$ are unknown. Hence, estimates of these probabilities, $\hat{p}_k$, are needed to be evaluated based on observations in the experimentation phase as follows:
\begin{equation}
\hat{p}_k = \frac{\sum_{n_1 = 1}^N \sum_{n_2 = 1}^N \dots \sum_{n_K = 1}^N \mathds{1}\{ r_{k, n_k} \geq \boldsymbol{r}_{-k, n_{-k}} \} }{N^K},
\label{p_k_hat_estimate}
\end{equation}
where $\boldsymbol{r}_{-k, n_{-k}} = \big ( r_{1, n_1}, r_{2, n_2}, \dots, r_{k - 1, n_{k - 1}}, r_{k + 1, n_{k + 1}},$ $\dots, r_{K, n_K} \big )$ and rewards of different arms are assumed to be independent.
\begin{remark}
If rewards of different arms are dependent, instantaneous observations of all arms at the same time are needed for $N$ times and $\hat{p}_k$ is calculated as follows:
\begin{equation}
\hat{p}_k = \frac{\sum_{n = 1}^N \mathds{1}\{ r_{k, n} \geq \boldsymbol{r}_{-k, n} \} }{N}.
\label{p_k_dependent}
\end{equation}
\end{remark}
The OTE-MAB algorithm selects arm $\hat{k} = \underset{k}{\argmax} \ \hat{p}_k$ as the best arm in terms of rewarding the most with the highest probability in one-time exploitation. The one-time exploitation regret, $r(\Delta p)$, which is a specific case of Definition \ref{regret_definition}, is defined as follows, where $k^*$ is defined in Equation \eqref{MMMMM}:
\begin{equation}
r(\Delta p) = P \left ( p_{k^*} - p_{\hat{k}} \geq \Delta p \right ).
\label{regret}
\end{equation}
The OTE-MAB algorithm is summarized in Algorithm \ref{OTE-MAB}.
We next present a theorem on an upper bound of the minimum number of experiments needed to guarantee an upper bound on regret of Algorithm \ref{OTE-MAB}.

\begin{theorem}
For any $0 < \epsilon_r, \Delta p < 1,$ if each of the $K$ arms are experimented for $N \geq \frac{2 \ln \left ( \frac{2K}{\epsilon_r} \right )}{{\Delta p}^2}$ times in the experimentation phase, the one-time exploitation regret defined in Equation \eqref{regret} is bounded by $\epsilon_r$, i.e. $r(\Delta p) \leq \epsilon_r$.
Note that simultaneous exploration of the $K$ arms are required in the experimentation phase if arm rewards are dependent.
\label{theorem1}
\end{theorem}

\begin{proof}
Consider the Bernoulli random variables $B_k = \mathds{1}\{R_k \geq \boldsymbol{R}_{-k}\}$ and their unknown means $p_k = \mathds{E}[B_k] = P(R_k \geq \boldsymbol{R}_{-k})$ for $k \in \mathcal{K}$.
Possessing $N$ independent observations from each of the $K$ independent or dependent arms in the pure exploration phase, the confidence interval derived from Hoeffding's inequality for estimating $p_{k}$ based on Equation \eqref{p_k_hat_estimate} or Equation \eqref{p_k_dependent}  with confidence level $1 - 2e^{-\frac{a^2}{2}}$ has the property that
\begin{equation}
P \left (  \hspace{-0.05cm} p_k \in \left ( \hat{p}_k - \frac{a}{2\sqrt{N}}, \hat{p}_k + \frac{a}{2\sqrt{N}} \right )  \hspace{-0.1cm} \right ) \hspace{-0.05cm} \geq  \hspace{-0.05cm} 1 - 2e^{-\frac{a^2}{2}}, \forall k \in \mathcal{K}.
\label{confidence_interval}
\end{equation}
Note that for the case of dependent arms, there is an $N$-tuple containing the instantaneous observation of the $K$ arm rewards as $\left (r_{1, n}, r_{2, n}, \dots, r_{K, n} \right )$ for $n \in \{1, 2, \dots, N\}$, which is used for estimation of $\hat{p}_k$ in Equation \eqref{p_k_dependent}. On the other hand, for the case of independent arms, any of the $N^K$ orderings of the $N$ observations of the $K$ arm rewards can be used for estimation of $\hat{p}_k$ as is done in Equation \eqref{p_k_hat_estimate}.
However, $\left ( \hat{p}_k - \frac{a}{2\sqrt{N^K}}, \hat{p}_k + \frac{a}{2\sqrt{N^K}} \right )$ cannot be used as confidence interval with confidence level $1 - 2e^{-\frac{a^2}{2}}$. The reason is that, although $\hat{p}_k$ is derived from $N^K$ samples, not all those samples are independent, but
exactly $N$ of the $N^K$ samples are independent.
In fact, the observed independent rewards can be classified as $N$-tuples of the $K$ arm rewards with independent elements in $ N^{k - 1} \times (N - 1)^{k - 1} \times \cdots \times 1^{k - 1} = (N!)^{K - 1}$ different ways. None of such $N$-tuples has any priority over the other ones to estimate $p_k$, so $\hat{p}_k$ can be computed based on any of the $N$-tuples. The estimate of $p_k$ derived from any of those $N$-tuples is in $\left (p_k - \frac{a}{2\sqrt{N}}, p_k + \frac{a}{2\sqrt{N}} \right )$ with probability at least $1 - 2e^{-\frac{a^2}{2}}$, so the average of those estimations is again in the mentioned interval with probability at least $1 - 2e^{-\frac{a^2}{2}}$.
Note that the average of estimates of $p_k$ derived from all of the $(N!)^{K - 1}$ different $N$-tuples is equal to $\hat{p}_k$ derived from Equation \eqref{p_k_hat_estimate} due to the following reason. An element of an $N$-tuple is repeated for $((N - 1)!)^{K - 1}$ times in all $N$-tuples. Hence, averaging over the $\frac{(N!)^{K - 1} \cdot N}{((N - 1)!)^{K - 1}} = N^K$ number of distinct elements of $N$-tuples results in the same answer as the case of averaging the estimates of $p_k$ derived from all of $(N!)^{K - 1}$ different $N$-tuples.
As a result, $\frac{a}{2\sqrt{N}}$ can be used as the half width of the confidence interval for estimators obtained from Equations \eqref{p_k_hat_estimate} and \eqref{p_k_dependent} for both independent and dependent arms.

In order to find a bound on regret, defined in Equation \eqref{regret} as $r(\Delta p) = P \left ( p_{k^*} - p_{\hat{k}} \geq \Delta p \right )$, note that
\begin{equation*}
\begin{aligned}
& \left \{ p_{k^*} - p_{\hat{k}} \geq \Delta p \right \} \subseteq \\
& \left \{ \exists k \in \mathcal{K} \text{ such that } p_k \notin \left ( \hat{p}_k - \frac{\Delta p}{2}, \hat{p}_k + \frac{\Delta p}{2} \right ) \right \} \overset{(a)}{\subseteq}
\end{aligned}
\end{equation*}
\begin{equation}
\begin{aligned}
& \left \{ \exists k \in \mathcal{K} \text{ such that } p_k \notin \left ( \hat{p}_k - \frac{a}{2\sqrt{N}}, \hat{p}_k + \frac{a}{2\sqrt{N}} \right ) \right \},
\end{aligned}
\end{equation}
where $(a)$ is true if $\frac{a}{2 \sqrt{N}} \leq \frac{\Delta p}{2}$. By using union bound and Equation \eqref{confidence_interval}, the probability of the right-hand side of the above equation can be bounded as follows, which results in the following bound on regret:
\begin{equation}
r(\Delta p) = P \left ( p_{k^*} - p_{\hat{k}} \geq \Delta p \right ) \leq 2Ke^{-\frac{a^2}{2}} = \epsilon_r.
\end{equation}

The above upper bound on regret is derived under the condition that $\frac{a}{2 \sqrt{N}} \leq \frac{\Delta p}{2}$, which by using $a^2 = 2 \ln \left ( \frac{2K}{\epsilon_r} \right )$ and simple algebraic calculations is equivalent to $N \geq \frac{2 \ln \left ( \frac{2K}{\epsilon_r} \right )}{{\Delta p}^2}$.
\end{proof}

According to Theorem \ref{theorem1}, the selected arm by Algorithm \ref{OTE-MAB}, $\hat{k}$, satisfies $p_{\hat{k}} = P \left ( R_{\hat{k}} \geq \boldsymbol{R}_{-\hat{k}} \right ) \geq (p_{k^*} - \Delta p)$ with probability at least $1 - \epsilon_r$ for any $0 < \epsilon_r, \Delta p < 1$, if each of the $K$ arms is explored in the experimentation phase for $N \geq \frac{2 \ln \left ( \frac{2K}{\epsilon_r} \right )}{{\Delta p}^2}$ times. Hence, $p_{\hat{k}}$ can get arbitrarily close to $p_{k^*}$ by increasing the number of pure explorations in the experimentation phase.

Let $p_{(1)}, p_{(2)}, \dots, p_{(K)}$ be the ordered list of $p_1,$ $p_2,$ $\dots,$ $p_K$ in descending order.
Note that arm $(1)$ is actually arm $k^*$ defined in Equation \eqref{MMMMM}.
Define the difference between the two maximum $p_k$'s as $\Delta p^* = p_{(1)} - p_{(2)}$, where without loss of generality is assumed to be nonzero. Having the knowledge of $\Delta p^*$ or a lower bound on it, a stronger notion of regret can be defined as
\begin{equation}
r = \underset{\Delta p > 0}{\inf} r(\Delta p) = P \left ( \hat{k} \neq k^* \right ),
\label{regret_optimum}
\end{equation}
and have the following corollary.

\begin{corollary}
From the theoretical point of view, upon the knowledge of $\Delta p^*$ or a lower bound on it, for any $0 < \epsilon_r < 1$, the regret defined in Equation \eqref{regret_optimum} is bounded by $\epsilon_r,$ i.e. $r < \epsilon_r$, if the $K$ arms are explored for $N \geq \frac{2 \ln \left ( \frac{2K}{\epsilon_r} \right )}{{\Delta p^*}^2}$ times each.
If arms are dependent, instantaneous explorations of the $K$ arms are needed.
\end{corollary}



\subsection{The FTE-MAB Algorithm}
\label{sub_FTE_MAB}
Consider the case where an arm is going to be exploited for finite number of times, $M < \infty$.
The best arm for $M$-time exploitations is defined in Equation \eqref{k_M_star}. Since reward distributions are unknown, $p_k^M$'s are needed to be estimated based on observations in pure exploration phase. In the case of independent arms, define the vector $\mathcal{R}_k^M$ with cardinality $N \choose M$ as
\begin{equation}
\mathcal{R}_k^M \hspace{-0.35mm} = \hspace{-0.35mm} \left \{ \sum_{n \in S_\mathcal{K}} \hspace{-0.6mm} r_{k, n} \text{ s.t. } S_\mathcal{K} \hspace{-0.3mm} \subseteq \hspace{-0.3mm} \{1, 2, \dots, N\} \text{ and } |S_\mathcal{K}| \hspace{-0.3mm} = \hspace{-0.3mm} M \right \} \hspace{-0.35mm} .
\end{equation}
Let $r_{k, j}^M$ for $1 \leq j \leq {N \choose M}$ be the different elements of $\mathcal{R}_k^M$. Let $\hat{p}_k^M$ be the estimate of $p_k^M$, where they can be computed as
\begin{equation}
\hat{p}_k^M = \frac{\sum_{j_1 = 1}^{N \choose M} \sum_{j_2 = 1}^{N \choose M} \dots \sum_{j_K = 1}^{N \choose M} \mathds{1}\{ r_{k, j_k}^M \geq \boldsymbol{r}_{-k, j_{-k}}^M \} }{{N \choose M}^K}.
\label{p_k_hat_estimate_M}
\end{equation}
In the case of dependent arms, $r_{k, j}^M$'s are defined in the same way as independent arms, but note that the set $S_\mathcal{K}$ corresponding to $r_{k, j}^M$ is used for generating $r_{k', j}^M$ for all $k' \in \mathcal{K}$. Hence, $\hat{p}_k^M$ is defined as follows for dependent arms:
\begin{equation}
\hat{p}_k^M = \frac{\sum_{j = 1}^{N \choose M} \mathds{1}\{ r_{k, j}^M \geq \boldsymbol{r}_{-k, j}^M \} }{{N \choose M}}.
\label{p_k_hat_estimate_MMM}
\end{equation}
The FTE-MAB algorithm selects arm $\hat{k} = \underset{k}{\argmax} \ \hat{p}_k^M$ for $M$-time exploitations. This algorithm is summarized in Algorithm \ref{FTE-MAB}.
We next present a theorem for an upper bound of the minimum number of experiments needed to guarantee an upper bound on regret of Algorithm \ref{FTE-MAB} which is the generalization of Theorem \ref{theorem1}.


\begin{algorithm}[t]
\caption{The FTE-MAB Algorithm}
\begin{algorithmic} 
\STATE \textbf{Input} $0 < \epsilon_r, \Delta p < 1$
\STATE choose $N$ such that $\lfloor \frac{N}{M} \rfloor \geq \frac{2 \ln \left ( \frac{2K}{\epsilon_r} \right )}{{\Delta p}^2}$
\STATE \textbf{Experimentation Phase:}
\bindent
\FOR{$n = 1$ to $N$}
\STATE $r_{k, n}$ is observed for all $k \in \mathcal{K}$
\ENDFOR
\eindent
\STATE Let $\mathcal{R}_k^M = \big \{ \sum_{n \in S_\mathcal{K}}$ $r_{k, n} \text{ s.t. } S_\mathcal{K} \subseteq \{1, 2, \dots, N\}$ $\text{ and }$ $|S_\mathcal{K}| = M \big \},$ where $r_{k, j}^M$ for $1 \leq j \leq {N \choose M}$ are the different elements of $\mathcal{R}_k^M$. Let the set $S_\mathcal{K}$ corresponding to $r_{k, j}^M$ be used for generating $r_{k', j}^M$ for all $k' \in \mathcal{K}$.
\IF{arms are independent}
\STATE Calculate $\hat{p}_k^M = \frac{\sum_{j_1 = 1}^{N \choose M} \sum_{j_2 = 1}^{N \choose M} \dots \sum_{j_K = 1}^{N \choose M} \mathds{1}\{ r_{k, j_k}^M \geq \boldsymbol{r}_{-k, j_{-k}}^M \} }{{N \choose M}^K}$
\ELSE
\STATE Calculate $\hat{p}_k^M = \frac{\sum_{j = 1}^{N \choose M} \mathds{1}\{ r_{k, j}^M \geq \boldsymbol{r}_{-k, j}^M \} }{{N \choose M}}$
\ENDIF
\STATE \textbf{Finite-Time Exploitation:}
\bindent
\STATE Play arm $\hat{k} = \underset{k}{\argmax} \ \hat{p}_k^M$.
\eindent
\end{algorithmic}
\label{FTE-MAB}
\end{algorithm}

\begin{theorem}
For any $0 < \epsilon_r, \Delta p < 1,$ if each of the $K$ arms is explored for $N$ times in the experimentation phase such that $\lfloor \frac{N}{M} \rfloor \geq \frac{2 \ln \left ( \frac{2K}{\epsilon_r} \right )}{{\Delta p}^2}$, the finite-time exploitation regret defined in Definition \ref{regret_definition} is bounded by $\epsilon_r$, i.e. $r_M(\Delta p) \leq \epsilon_r$. If the rewards of different arms are dependent, simultaneous explorations of the $K$ arms are required for the same bound on regret.
\label{theorem2}
\end{theorem}

The proof of Theorem \ref{theorem2} is similar to that of Theorem \ref{theorem1}, which can be found in the Appendix.


Let $p_{(1)}^M, p_{(2)}^M, \dots, p_{(K)}^M$ be the ordered list of $p_1^M,$ $p_2^M, \dots, p_K^M$ in descending order.
Note that arm $(1)$ is actually arm $k^*$ defined in Equation \eqref{OTE-MAB}.
Define the difference between the two maximum $p_k^M$'s as $\Delta p^*_M = p_{(1)}^M - p_{(2)}^M$, where without loss of generality is assumed to be nonzero. Having the knowledge of $\Delta p^*_M$ or a lower bound on it, a stronger notion of regret can be defined as
\begin{equation}
r_M = \underset{\Delta p > 0}{\inf} r_M(\Delta p) = P \left ( \hat{k} \neq k^* \right ),
\label{regret_optimum_M}
\end{equation}
and have the following corollary.

\begin{corollary}
From the theoretical point of view, upon the knowledge of $\Delta p^*_M$ or a lower bound on it, for any $0 < \epsilon_r < 1$, the regret defined in Equation \eqref{regret_optimum_M} is bounded by $\epsilon_r,$ i.e. $r_M < \epsilon_r$, if the $K$ arms are explored for $N$ times each, where $ \lfloor \frac{N}{M} \rfloor \geq \frac{2 \ln \left ( \frac{2K}{\epsilon_r} \right )}{{\Delta p^*_M}^2}$.
If arms are dependent, instantaneous explorations of the $K$ arms are needed.
\end{corollary}

\begin{corollary}
If $M$ converges to infinity, the problem becomes the classical multi-armed bandit problem since $\underset{k}{\argmax} \ P(R_k^M \geq \boldsymbol{R}_{-k}^M)$ is the same as $\underset{k}{\argmax} \ P \left ( \frac{R_k^M}{M} \geq \frac{\boldsymbol{R}_{-k}^M}{M} \right )$ and due to the law of large numbers $\frac{R_k^M}{M} \rightarrow \mu_k$ as $M \rightarrow \infty$. Hence, the FTE-MAB algorithm selects the arm with maximum expected reward if the arm is going to be exploited for infinitely many times and the cumulative reward is desired to be maximized.
\end{corollary}



\begin{figure}[t]
\centering
\includegraphics[width=0.485\textwidth]{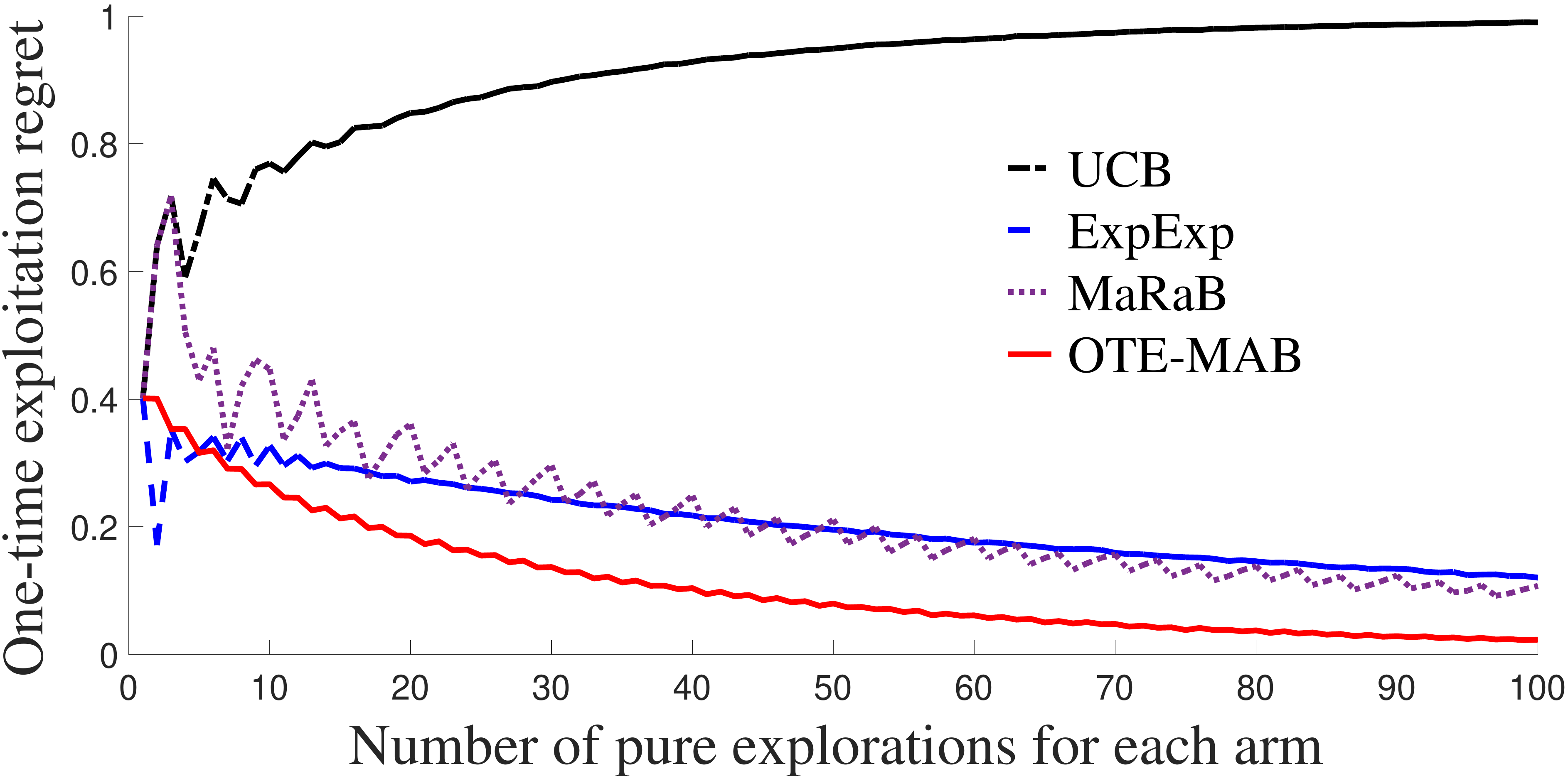}
\caption{Comparison of regret for OTE-MAB against the state-of-the-art algorithms for example \ref{example1}.}
\label{result1}
\end{figure}

\section{Simulation Results}
\label{results}
In this section, we report numerical simulations validating the theoretical results presented in this paper. We compare our proposed OTE-MAB algorithm with the Upper Confidence Bound (UCB) \cite{auer2002finite}, ExpExp \cite{sani2012risk}, and MaRaB \cite{galichet2013exploration} algorithms. Consider two arms with the reward distributions given in example \ref{example1}. The regret defined in Equation \eqref{regret_optimum} versus the number of pure explorations for each arm, $N$, is averaged over 100,000 runs. The result is plotted in Figure \ref{result1} and as is shown OTE-MAB outperforms the state-of-the-art algorithms for the purpose of risk-aversion in terms of the regret defined in this paper.
Note that the UCB algorithm aims at selecting an arm that maximizes the expected received reward, but in example \ref{example1}, the arm with higher expected reward is less probable to have the highest reward, which is why the UCB algorithm performs poorly in this example.
However, in the following example where the arm that rewards more on expectation is also more probable to reward more, the UCB, ExpExp, and MaRaB algorithms perform as well as the OTE-MAB algorithm.
\begin{example}
Consider two arms with the following unknown independent reward distributions:
\[
\begin{aligned}
&f_1(u) = \alpha e^{-0.5(u - 2)^2} \cdot \mathds{1} \{ 0 \leq u \leq 10 \} \\
& f_2(u) = \beta e^{-0.5(u - 1)^2} \cdot \mathds{1} \{ 0 \leq u \leq 10 \},
\end{aligned}
\]
where $\alpha$ and $\beta$ are constants so that the two probability distribution functions integrate to one.
\label{example2}
\end{example}

Note that in example \ref{example2}, $\mathds{E}[R_1] > \mathds{E}[R_2]$ and $P(R_1 \geq R_2) > 0.5$. For this scenario, the regret defined in Equation \eqref{regret_optimum} versus the number of pure explorations for each arm, $N$, averaged over 100,000 runs is plotted in Figure \ref{result2}.

\begin{figure}[t]
\centering
\includegraphics[width=0.485\textwidth]{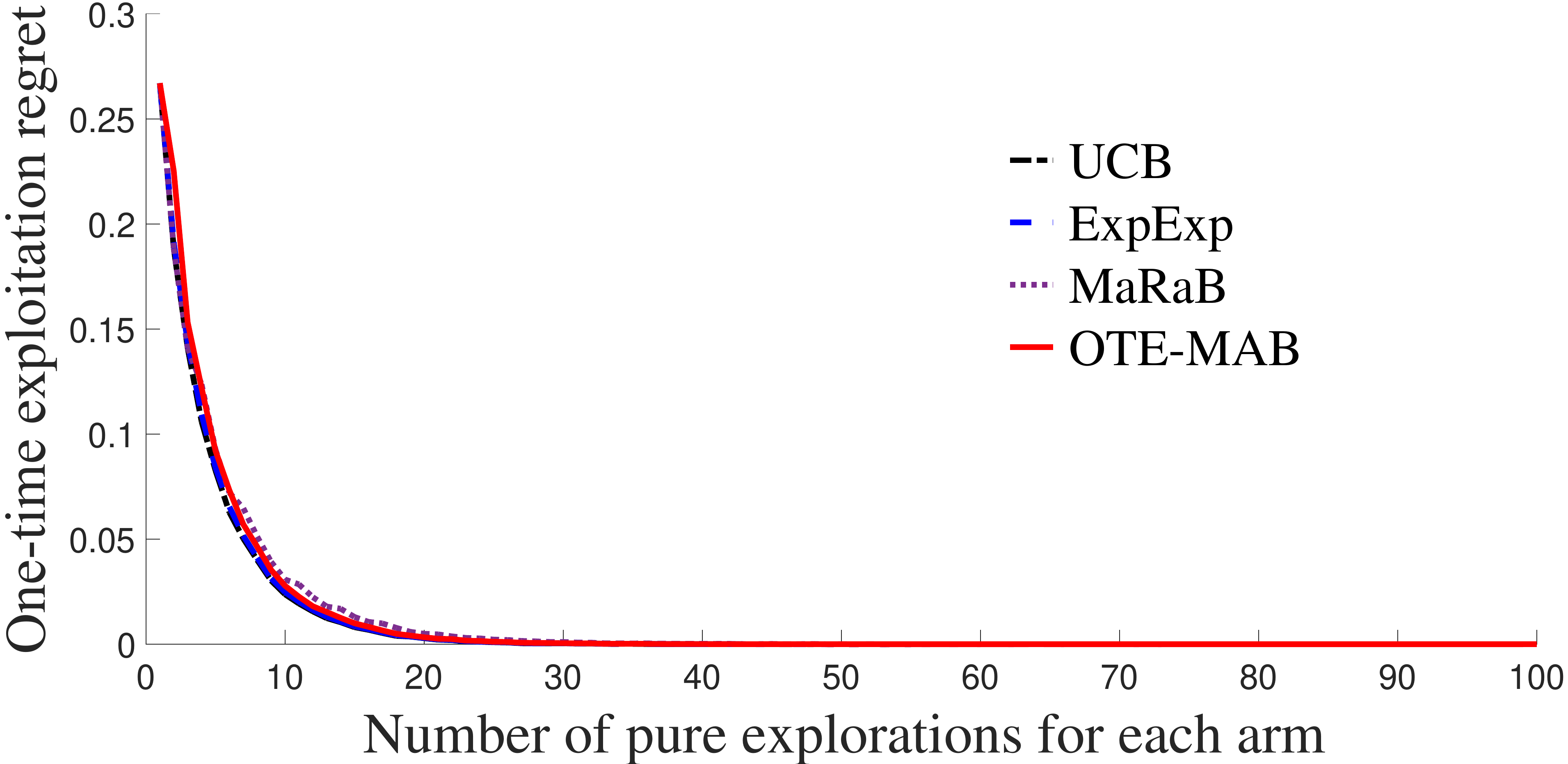}
\caption{Comparison of regret for OTE-MAB against the state-of-the-art algorithms for example \ref{example2}.}
\label{result2}
\end{figure}

In another experiment, the multi-armed bandit is simulated for example \ref{example1} and the probability that the selected arm has the higher reward is calculated over 500,000 runs for different algorithms. The result is shown in Figure \ref{result_probability}. This result confirms the motivation of our study on risk-averse finite-time exploitations in multi-armed bandits.
\begin{figure}[t]
\centering
\includegraphics[width=0.485\textwidth]{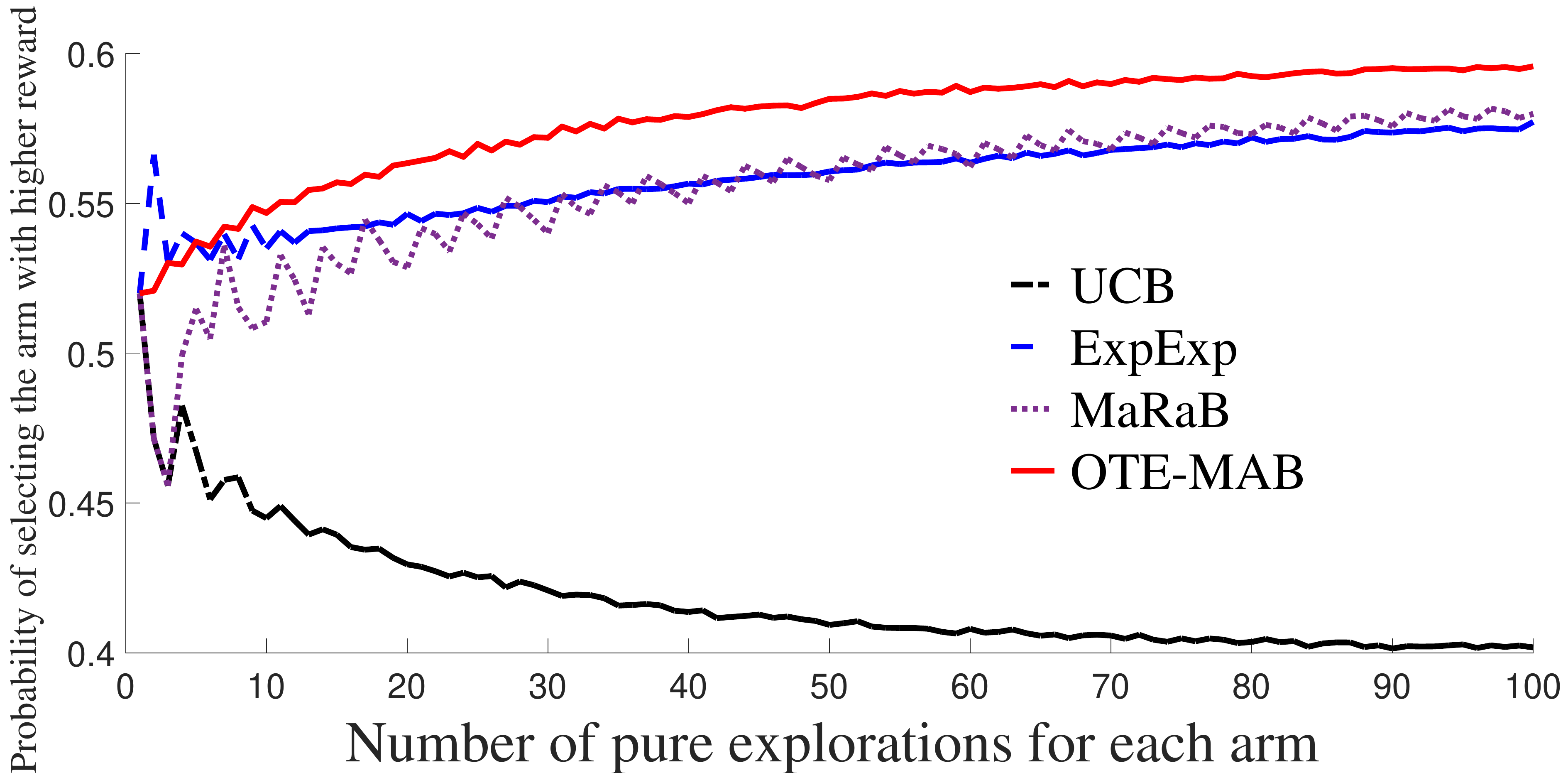}
\caption{Comparison of probability of selecting the arm with higher reward for OTE-MAB against the state-of-the-art algorithms for example \ref{example1}.}
\label{result_probability}
\end{figure}

In the above comparison of OTE-MAB with state-of-the-art algorithms, three different choices of hyper-parameters for the ExpExp and MaRaB algorithms are tested and the best performance is presented. However, note that the performances of these algorithms depend on the choice of hyper-parameter. In Figure \ref{result3}, the sensitivity of the performance of ExpExp algorithm with respect to the choice of hyper-parameter $\rho$ is depicted for example \ref{example1} and a third example where the variance of the best arm is larger than the variance of the arm with lower expected reward. The two plots are the averaged regret over 100,000 runs versus the value of $\rho$ for the ExpExp algorithm for two different multi-armed bandit problems when $N = 100$. As depicted in Figure \ref{result3}, a choice of $\rho$ can be good for one multi-armed bandit problem, but not good for another one.
Due to our observations, the sensitivity of the MaRaB algorithm to its hyper-parameter can even be more complex. Figure \ref{result4} depicts the averaged regret over 100,000 runs versus the value of MaRaB hyper-parameter, $\alpha$, when $N = 100$. This figure is plotted for example \ref{example1} and a fourth example where reward of the first arm has a truncated Gaussian distribution with mean three and variance two over the interval $[0, 10]$ and the second arm is the same as the one in example \ref{example1}.

\begin{figure}[t]
\centering
\includegraphics[width=0.485\textwidth]{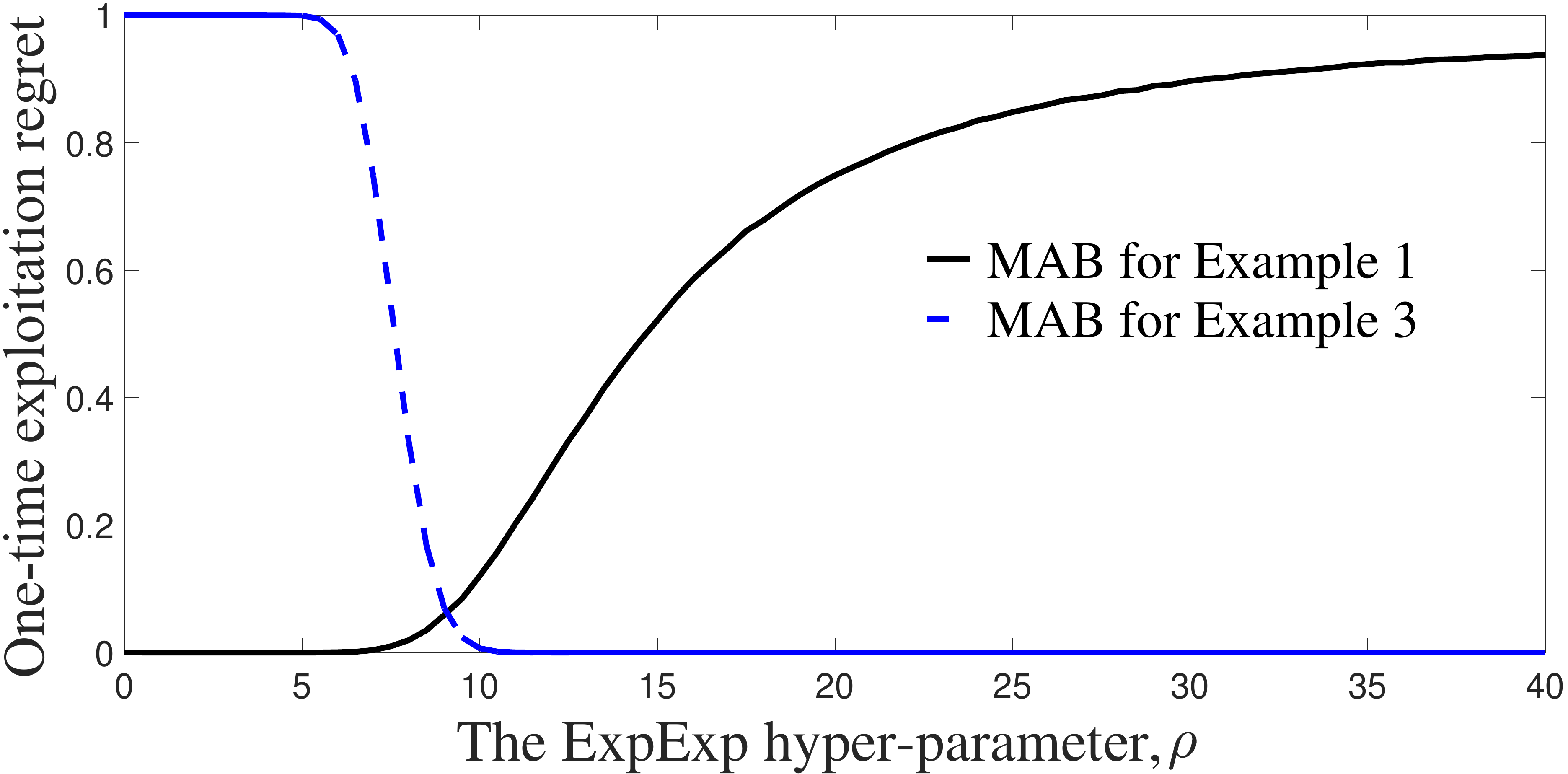}
\caption{Regret of the ExpExp algorithm versus the hyper-parameter $\rho$ for two examples.}
\label{result3}
\end{figure}

\begin{figure}[t]
\centering
\includegraphics[width=0.485\textwidth]{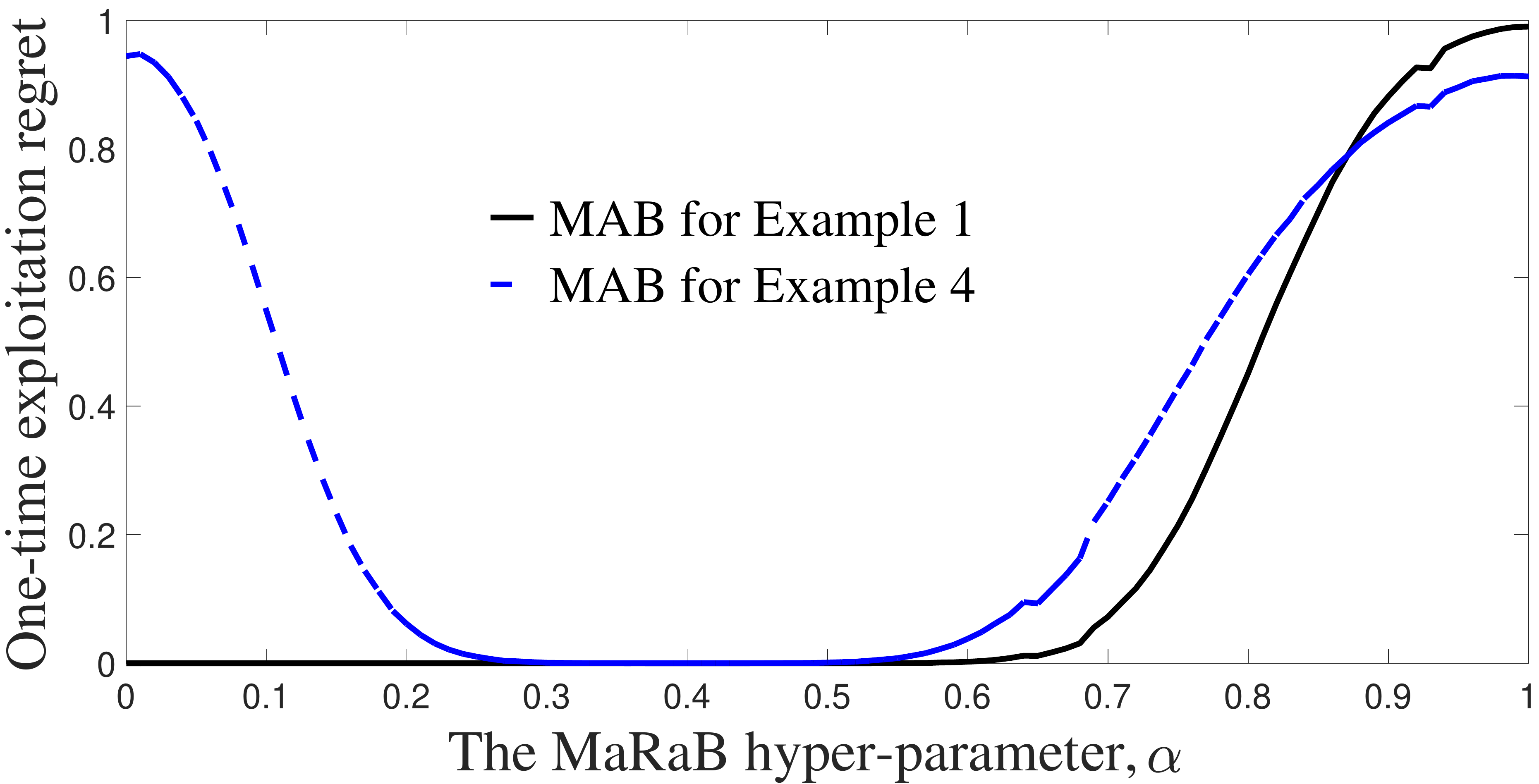}
\caption{Regret of the MaRaB algorithm versus the hyper-parameter $\alpha$ for two examples.}
\label{result4}
\end{figure}

In another experiment, we compare the minimum number of explorations needed to guarantee a bound on regret for two cases of one-time and two-time exploitations.
Theorems \ref{theorem1} and \ref{theorem2} suggest that for given $K, \epsilon_r$, and $\Delta p^* = \Delta p^*_M$, the upper bound of minimum number of explorations needed for $M$-time exploitations to guarantee that the regret is bounded by $\epsilon_r$ is $M$ times that of one-time exploitation.
We design two examples of two-armed bandits such that $\Delta p^* = \Delta p^*_2 = 0.28$ and plot the minimum number of explorations to guarantee bounded regret by $\epsilon_r$ in Figure \ref{figure_minimum}. The dashed line is the plot of the OTE-MAB algorithm multiplied by two which is close to the one related to the FTE-MAB algorithm for two-armed bandits. This observation supports our theoretical results.

\begin{figure}[t]
\centering
\includegraphics[width=0.485\textwidth]{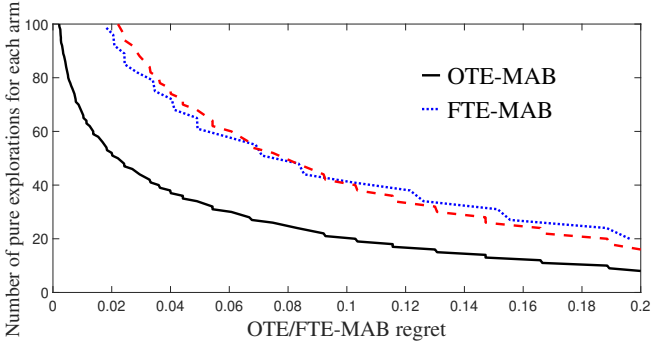}
\caption{The minimum number of explorations needed to guarantee a bound on regret for two cases of one-time and two-time exploitations.}
\label{figure_minimum}
\end{figure}

Taking a closer look at example \ref{example1}, we note that the regret defined in \eqref{regret_optimum} can generally be formulated as
\begin{equation}
\begin{aligned}
r^* = & \sum_{i = \lfloor \frac{N}{2} \rfloor + 1}^{N} {N \choose i} \cdot (1 - p_{k^*})^i \cdot p_{k^*}^{N - i} + \\
 & \frac12 \cdot {N \choose \frac{N}{2}} \cdot (1 - p_{k^*})^{\frac{N}{2}} \cdot p_{k^*}^{\frac{N}{2}} \cdot \mathds{1} \{ \text{$N$ is even} \}.
\end{aligned}
\label{exact_regret}
\end{equation}
Deriving the regret from the above equation, the same regret is found as the one generated by simulation for the OTE-MAB algorithm that is plotted in Figure \ref{result1}.
In this paper, the experimentation is assumed to have zero cost, which is often a valid assumption. However, if experimentation is time-consuming, there is a cost to postpone the exploitation of the best identified arm. For example, for more experimentation, a patient receives medication by delay or an investor keeps his/her money on hold with zero interest, both of which incur costs. Let such a cost be formulated by an increasing function $C(.)$, where $C(N)$ is the incurred cost of $N$ experiments. Then, a trade-off between more exploration for higher accuracy of best-arm identification and lower incurred cost of experimentation emerges. Such a trade-off can be formalized by solving
\begin{equation}
\underset{N}{\argmin} \ C(N) + \alpha \cdot \bar{r}^*,
\label{argmin_trade_off}
\end{equation}
where $\alpha$ is the cost-regret trade-off and $\bar{r}^*$ is calculated by Equation \eqref{exact_regret} based on an estimation of $p_{k^*}$ which is updated after each experiment. Figure \ref{trade_off} plots $C(N) + \alpha \cdot r^*$ under example \ref{example1} for $C(N) = \frac{N}{5}, \alpha = 100$, and the real value of $r^*$. The rigorous analysis of \eqref{argmin_trade_off} is postponed for future work \cite{yekkehkhany2019cost} and is beyond the scope of this paper.

\begin{figure}[t]
\centering
\includegraphics[width=0.49\textwidth]{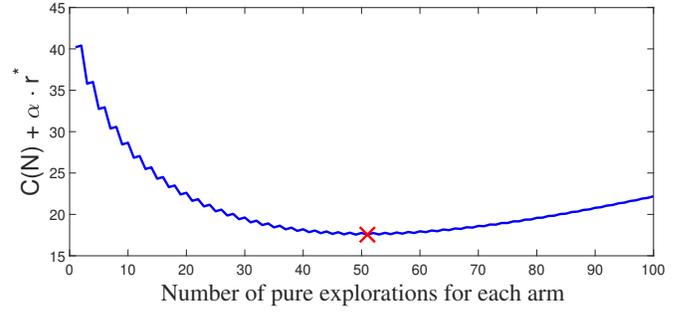}
\caption{Cost-regret trade-off is addressed by minimizing a linear combination of cost and regret.}
\label{trade_off}
\end{figure}

\section{Conclusion and Future Work}
\label{conclusion_future}
The focus of this work is on application domains, such as personalized health-care and one-time investment, where an experimentation phase of pure arm exploration is followed by a given finite number of exploitations of the best identified arm. We show through an example that the arm with maximum expected reward does not necessarily maximize the probability of receiving the maximum reward. The OTE-MAB and FTE-MAB algorithms are presented in this paper whose goals are to select the arm that maximizes the probability of receiving the maximum reward. We define a new notion of regret for our problem setup and find an upper bound on the minimum number of experiments that should be done to guarantee an upper bound on regret.
The cost of experimentation is assumed to be negligible in this paper, but if such an assumption is violated in an application domain, one can study the cost-regret trade-off as a promising future work in various deterministic and stochastic versions.



\bibliographystyle{IEEEtran}
\bibliography{sigproc}

\appendix
\textit{Proof of Theorem \ref{theorem2}:}
Consider the Bernoulli random variables $B_k^M = \mathds{1}\{R_k^M \geq \boldsymbol{R}_{-k}^M\}$ and their unknown means $p_k^M = \mathds{E}[B_k^M] = P(R_k^M \geq \boldsymbol{R}_{-k}^M)$ for $k \in \mathcal{K}$.
Possessing $N$ independent observations from each of the $K$ independent or dependent arms in pure exploration, there are exactly $\lfloor \frac{N}{M} \rfloor$ independent samples for estimation of $p_k^M$.
Due to the same reasoning in the proof of Theorem \ref{theorem1}, the confidence interval for estimating $p_{k}^M$ based on Equation \eqref{p_k_hat_estimate_M} or \eqref{p_k_hat_estimate_MMM} with confidence level $1 - 2e^{-\frac{a^2}{2}}$ has the property that
\begin{equation}
P \left ( \hspace{-0.05cm} p_k^M \hspace{-0.05cm} \in \hspace{-0.05cm} \left ( \hat{p}_k^M - \frac{a}{2\sqrt{\lfloor \frac{N}{M} \rfloor}}, \hat{p}_k^M + \frac{a}{2\sqrt{\lfloor \frac{N}{M} \rfloor}} \right ) \hspace{-0.1cm} \right ) \hspace{-0.05cm} \geq \hspace{-0.05cm} 1 - 2e^{-\frac{a^2}{2}},
\label{confidence_interval}
\end{equation}
for all $k \in \mathcal{K}$.

In order to find a bound on regret, defined in Definition \ref{regret_definition} as $r_M(\Delta p) = P \left ( p_{k^*}^M - p_{\hat{k}}^M \geq \Delta p \right )$, note that
\begin{equation}
\begin{aligned}
& \left \{ p_{k^*}^M- p_{\hat{k}}^M \geq \Delta p \right \} \subseteq \\
& \left \{ \exists k \in \mathcal{K} \text{ s.t. } p_k^M \notin \left ( \hat{p}_k^M - \frac{\Delta p}{2}, \hat{p}_k^M + \frac{\Delta p}{2} \right ) \right \} \overset{(a)}{\subseteq} \\
& \left \{ \exists k \in \mathcal{K} \text{ s.t. } p_k^M \hspace{-0.5mm} \notin \hspace{-0.5mm} \left ( \hspace{-1mm} \hat{p}_k^M - \frac{a}{2\sqrt{\lfloor \frac{N}{M} \rfloor}}, \hat{p}_k^M + \frac{a}{2\sqrt{\lfloor \frac{N}{M} \rfloor}} \right ) \hspace{-1mm} \right \},
\end{aligned}
\end{equation}
where $(a)$ is true if $\frac{a}{2 \sqrt{\lfloor \frac{N}{M} \rfloor}} \leq \frac{\Delta p}{2}$. By using union bound and Equation \eqref{confidence_interval}, the probability of the right-hand side of the above equation can be bounded as follows, which results in the following bound on regret:
\begin{equation}
r_M(\Delta p) = P \left ( p_{k^*}^M - p_{\hat{k}}^M \geq \Delta p \right ) \leq 2Ke^{-\frac{a^2}{2}} = \epsilon_r.
\end{equation}

The above upper bound on regret is derived under the condition that $\frac{a}{2 \sqrt{\lfloor \frac{N}{M} \rfloor}} \leq \frac{\Delta p}{2}$, which by using $a^2 = \frac{K}{\epsilon_r}$ and simple algebraic calculations is equivalent to $\lfloor \frac{N}{M} \rfloor \geq \frac{2 \ln \left ( \frac{2K}{\epsilon_r} \right )}{{\Delta p}^2}$.

\end{document}